\documentclass[11pt]{article}

\usepackage[utf8]{inputenc}
\usepackage[T1]{fontenc}
\usepackage{microtype}
\usepackage{fullpage}
\usepackage[osf,sc]{mathpazo}
\RequirePackage[scaled=0.90]{helvet}
\RequirePackage[scaled=0.85]{beramono}
\RequirePackage{textcomp}
\usepackage{graphicx}
\usepackage{booktabs}
\usepackage{nicefrac}
\usepackage{url}
\usepackage{xcolor}
\definecolor{DarkGreen}{rgb}{0.075,0.375,0.075}
\definecolor{DarkRed}{rgb}{0.5,0.1,0.1}
\definecolor{DarkBlue}{rgb}{0.1,0.1,0.5}
\usepackage{natbib}
\usepackage[skip=2mm,indent=5mm]{parskip}
\usepackage[small]{caption}
\usepackage{amsthm,amsmath,amsfonts,amssymb}
\usepackage{mathtools}
\usepackage{enumitem}
\usepackage{algorithm}
\usepackage{algpseudocode}
\usepackage{thm-restate}
\usepackage[pdftex]{hyperref}
\hypersetup{colorlinks=true, linkcolor=DarkBlue, citecolor=DarkGreen, filecolor=DarkRed, urlcolor=DarkBlue,
            pdftitle={Circumventing the Computational Hardness of Computational Arbitrage},
            pdfauthor={Nastaran Okati, Ricardo Olmedo, Moritz Hardt}}
\usepackage{cleveref}

\graphicspath{{figures/}}

\usepackage{amsmath,amsfonts,bm}

\def\eqref#1{equation~\ref{#1}}

\def\1{\bm{1}}

\def\vc{{\bm{c}}}

\def\vy{{\bm{y}}}

\DeclareMathAlphabet{\mathsfit}{\encodingdefault}{\sfdefault}{m}{sl}
\SetMathAlphabet{\mathsfit}{bold}{\encodingdefault}{\sfdefault}{bx}{n}

\def\gD{{\mathcal{D}}}

\def\gX{{\mathcal{X}}}
\def\gY{{\mathcal{Y}}}

\newcommand{\E}{\mathbb{E}}

\newcommand{\R}{\mathbb{R}}

\DeclareMathOperator*{\argmax}{arg\,max}

\newcommand{\pass}[3]{\operatorname{pass}_{#1,#2}@#3}
\newcommand{\hatpass}[3]{\widehat{\operatorname{pass}_{#1,#2}@#3}}

\newtheorem{theorem}{Theorem}

\newtheorem{assumption}{Assumption}
\theoremstyle{definition}
\newtheorem{example}{Example}

\newcommand{\ie}{i.e.}

\title{Circumventing the Computational Hardness of Computational Arbitrage}
\author{%
  Nastaran Okati\footnote{Corresponding author: \href{mailto:nastaran.okati@tuebingen.mpg.de}{nastaran.okati@tuebingen.mpg.de}} , 
  Ricardo Olmedo, and Moritz Hardt}
\date{\textit{Max Planck Institute for Intelligent Systems, and T\"ubingen AI Center}}

\begin{document}
\maketitle

\begin{abstract}
AI model marketplaces host providers with heterogeneous
costs and capabilities. These differences create an opportunity for \emph{computational arbitrage}~\citep{olmedo2026computational}, in which an intermediary purchases inference from a portfolio of models and resells it to users at a profit. Beyond creating profit opportunities for the arbitrageur, computational arbitrage offers financial value to both consumers and providers and could therefore significantly influence AI model markets.
While prior work empirically demonstrates the feasibility of computational arbitrage for small model pools, scaling it to large marketplaces requires arbitrage strategies that are both computationally tractable and inexpensive to identify and deploy.
In this work, we formalize computational arbitrage as an allocation problem and prove that it is generally NP-hard. However, we show that the optimal allocation can be computed efficiently under a natural diminishing-returns condition on how model performance scales with inference spending. We further develop a sample-efficient method for estimating model performance from limited data, reducing the upfront cost of identifying arbitrage strategies.
Our analysis also yields several economic insights: at the optimum, every funded model delivers the same marginal utility per dollar; arbitrage directs revenue away from providers offering little utility relative to their prices; and it rewards complementary specialization.  
\end{abstract}

\section{Introduction}
\begin{figure}[t]
    \centering
    \includegraphics[width=0.9\textwidth]{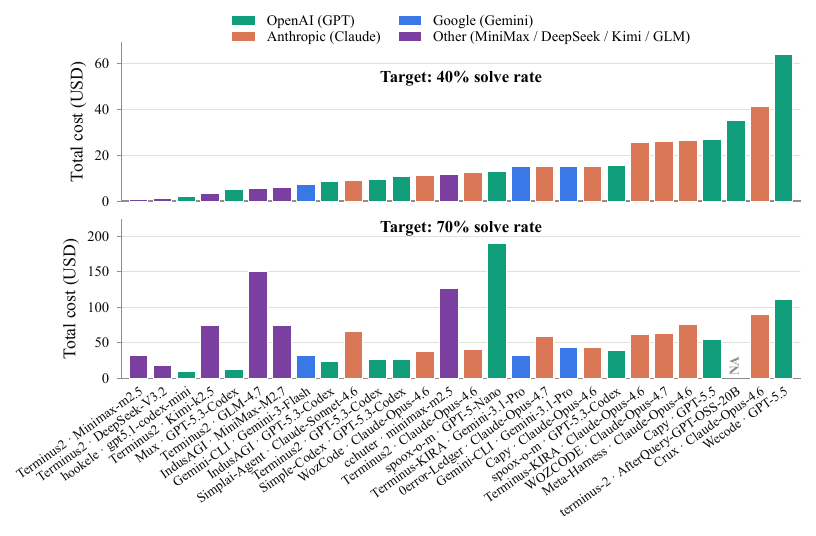}
    \caption{\textbf{The cost of a given performance target varies by more than an order of magnitude
across the market.} Expected total cost for each model to reach a $40\%$ (top) and a $70\%$ (bottom) solve rate on Terminal-Bench~2.0. The figure also shows that the cheapest models differ between the $40\%$ and $70\%$ targets, indicating that relative cost-effectiveness varies with the desired solve rate.
    For each model, we give every query the same budget, compute the probability that repeated sampling solves it before the budget is exhausted, and increase the budget until the average solve rate reaches the target. 
    ``NA'' marks a model that does not reach the target within the budget range considered. 
    Prices are based on the reported cost per trial when available, and otherwise token counts scaled by the LiteLLM price database as of September 2026. 
    Refer to \Cref{app:fig_1} for experimental details.
    } 
    \label{fig:tbench_cost} 
\end{figure}  

The proliferation of AI models with heterogeneous costs and capabilities makes it increasingly difficult for users to decide which service to choose. Even among models offering comparable performance, inference costs can differ substantially. For example, reaching a $40\%$ solve rate on Terminal-Bench can cost anywhere from less than a dollar to more than $\$60$; reaching $70\%$ can cost from roughly $\$10$ to over $\$180$ (see \Cref{fig:tbench_cost}). 
Such price dispersion creates the potential for \emph{computational arbitrage}, where an intermediary allocates inference budget across model providers, achieves a given level of performance at a lower cost, and resells it to users at a profit.
%
Computational arbitrage brings broad economic benefits, including lower consumer prices, reduced market segmentation, and easier market entry for smaller providers~\citep{olmedo2026computational}.

Despite these benefits, computational arbitrage is not yet deployed at scale in AI marketplaces, due to several challenges.
First, given model prices and capabilities, an arbitrageur must determine how to allocate an inference budget across models to maximize utility. As the number of models grows, exhaustively searching over possible allocations becomes computationally infeasible. Second, even after the budget is allocated, the arbitrageur must determine the order in which to query models to avoid unnecessary spending. Third, computing an effective allocation requires knowing how each model's success probability varies across queries and inference budgets, quantities that are generally unavailable and prohibitively expensive to estimate, as doing so requires many paid model calls. 

In this work, we address these challenges and show that the optimal arbitrage strategy is both tractable to compute and inexpensive to identify and deploy in large model marketplaces.
%
Our results yield several interesting economic implications. At the optimum, every funded model delivers the same marginal utility per dollar. Revenue therefore flows toward providers whose performance justifies their prices, while those offering little utility relative to their prices receive little or no revenue. This result contributes to the growing literature on LLM service pricing: regardless of the billing scheme used, a provider’s ability to attract revenue ultimately depends on the utility it delivers per dollar. 
Providers can therefore compete in two ways: by offering sufficiently low prices or by specializing in queries that other models are likely to fail; computational arbitrage hence creates an incentive for complementary specialization.

In more detail, we make the following contributions:
\begin{itemize}[itemsep=4pt, topsep=0pt, parsep=0pt, partopsep=0pt]
    \item \textbf{Optimal arbitrage strategy.} We formulate arbitrage as a budget-allocation problem and show that it is generally NP-hard. We then show that whenever inference spending has diminishing returns in log-failure reduction, the objective is jointly concave and its global optimum is computable efficiently even in large model markets. Across four agentic and coding benchmarks, our arbitrage strategy achieves a better cost–performance trade-off than any individual model
    (see \Cref{fig:frontier}).
    
    \item \textbf{Deployment order.} An arbitrageur can reduce their realized costs by querying funded models sequentially and stopping once an acceptable solution is found. Although we show that finding the optimal fixed order is NP-hard, we develop a greedy rule that prioritizes models with the highest expected utility per cost on queries missed by earlier models. 
    Across benchmarks, our greedy ordering closely tracks the best order found by exhaustive search
    (see \Cref{fig:order}).
    
    \item \textbf{Sample-efficient utility estimation.} 
    Building on distributional strategies in predicting inference-time scaling,
    we adapt empirical Bayes estimators that recover query-level utility curves from limited data. Across six benchmarks, an arbitrageur using our estimation method reaches nearly the same solve rate as an arbitrageur estimating the utility curves from orders of magnitude more data
    (see \Cref{fig:calibration,fig:calibration_app}).\footnote{Code and data are available at \url{https://github.com/nastaran75/arbitrage-markets}.}
    

\end{itemize}

\section{The optimal arbitrage strategy}\label{sec:optimal}
%
%
We consider a market of $N$ model providers where each model can map user queries, $x\in \gX$, to proposed solutions $y\in\gY$. The quality of the solution depends on the model, the difficulty of the query, as well as the amount of inference budget, $c$, devoted to the query. For example, the budget could represent the number of sampling attempts~\citep{chow} or reasoning effort~\citep{wei2022chain}. Each model $i$ hence induces a conditional distribution $p_i(y\mid x,c)$. 
Let $A:\gX\times\gY\to\{0,1\}$ indicate whether $y$ is an acceptable solution to query $x$.   
%
For a given query distribution $x\sim D$ and a per-query budget $C$, a rational user would spend the budget on the model that achieves the highest acceptance rate, \ie,    
$\argmax_{i\in[N]}
\E_{x\sim D, y\sim p_i(\cdot\mid x,C)}
\left[A(x,y)\right].
$
%
Surprisingly, this strategy can be suboptimal
even when that model outperforms every other model on every query. Specifically, diminishing returns to inference spending
can make allocating additional budget to a weaker model more
valuable than spending it on the dominant model.
Furthermore, allocating the budget across multiple models can exploit \emph{complementary} strengths and provide a further reason
to diversify.\footnote{\Cref{app:example_1,app:example_2} in \Cref{app:examples} provide simple two-model examples in which distributing the budget
among the two models strictly beats every single-model allocation of the same total budget under both scenarios.}

An arbitrageur is an intermediary who benefits from this gap through diversification.
%
%
%
Rather than allocating the entire budget to the most cost-effective model, the arbitrageur chooses a query-independent allocation 
$\vc=(c_1,\dots,c_N)\in\R_{\geq 0}^N$, such that $\sum_{i\in[N]} c_i \leq C$,
\ie, for every query $x$, an inference budget of $c_i$ is allocated to model $i$. 
%
Let $\vy=(y_1,\ldots,y_N)$ denote the outputs produced by the $N$ models under allocation $\vc$. We denote their joint distribution as
$p(\vy\mid x,\vc)$, with marginals
$p_i(y_i\mid x,c_i)$. The arbitrageur wants to maximize the acceptance probability, given the budget, \ie, 
\begin{align}\label{eq:cost_objective}
\underset{\vc\in\R_{\geq 0}^N}{\mathrm{maximize}} 
\quad
U(\vc)
:=
\E_{x\sim D,\vy\sim p(\cdot\mid x,\vc)}
\left[
1-\prod_{i=1}^N\bigl(1-A(x,y_i)\bigr)
\right]
~~\text{s.t.}
~
&\sum_{i\in[N]}c_i\leq C.
\end{align}

%

The combinatorial nature of the problem makes it computationally intractable:\footnote{All proofs can be found in \Cref{app:proofs}.} 

\begin{restatable}{theorem}{thmNPhard}
    \label{thm:arbitrage_nphard}
    Finding the optimal arbitrage policy as in \Cref{eq:cost_objective} is NP-hard.
\end{restatable}

%
%
We therefore identify conditions under which the problem becomes tractable. For each model, we first define the utility function $u_i: \gX \times \R_{\geq0} \to [0,1]$ as
   $ u_i(x;c_i):=\E_{y\sim p_i(\cdot\mid x,c_i)}[A(x,y)]$,
and assume that model outputs are conditionally independent given the query.\footnote{This assumption says that conditional on a query, the sampling processes of different LLMs are mutually independent, a natural assumption made also in prior work~(\citealp[Assumption~C.1]{ai2025beyond}; \citealp{farinhas2024reranking}). This assumption does not restrict arbitrary overlap in models' strengths and weaknesses \emph{across} queries.} We can hence rewrite the objective in \Cref{eq:cost_objective} as follows:
\begin{align}\label{eq:indep_cost_objective}
    U(\vc)
    &= \E_{x\sim D}\left[1-\prod_{i=1}^N (1-u_i(x;c_i))\right] = \E_{x\sim D} \left[1-e^{-\sum_{i=1}^N\alpha_i(x,c_i)}\right]
    ~~\text{s.t.}\sum_{i\in[N]}c_i\leq C,
\end{align}
where $\alpha_i(x;c):=-\log\bigl(1-u_i(x;c)\bigr)$, \ie, the log-failure reduction on query $x$ purchased with $c$ dollars from model $i$. Conditional independence of model outcomes given a query does not restore tractability;\footnote{In the reduction proving
\Cref{thm:arbitrage_nphard}, acceptance outcomes are
conditionally independent given the query.} we therefore impose a natural diminishing property on how model performance scales with inference spending; more formally we have the following assumption:
\begin{assumption}\label{ass:logconvex}
For every model $i\in[N]$ and query $x\in\gX$, the log-failure reduction $c \mapsto \alpha_i(x,c)$ is non-decreasing and concave on $\R_{\geq 0}$.
\end{assumption}
\Cref{ass:logconvex} formalizes diminishing returns to inference compute: it requires the log-failure reduction purchased from each additional dollar to be non-increasing. This property has been empirically observed for common inference-scaling strategies, including repeated sampling~\citep{schaeffer2025powerlaws}, chain-of-thought reasoning~\citep{ghosal2025thinking,snell2024scalingtesttime}, and post-training~\citep{chen2026leaderboard}.

This assumption makes the allocation problem tractable:
\begin{restatable}{theorem}{thmKKT}\label{thm:KKT}
Under \Cref{ass:logconvex}, $U$ is non-decreasing and jointly concave. If $D$ has finite support and each
$\alpha_i(x;\cdot)$ is differentiable, a global optimum exists and is characterized by the KKT conditions. This optimum can be found efficiently using cyclic coordinate-ascent on the Lagrangian.
\end{restatable}

We call the map \(C\mapsto U(\vc^\star(C))\), giving the best solve rate attainable at each budget, the \emph{arbitrage frontier}. We trace this frontier by sweeping the shadow price \(\lambda\) over a grid and maximizing the Lagrangian $U(\vc)-\lambda\sum_i c_i$ at each price by cyclic coordinate ascent; \Cref{alg:frontier} provides the pseudocode. Each coordinate-ascent sweep requires \(N\) coordinate updates, whereas exhaustive grid search, as in~\citet{olmedo2026computational}, scales exponentially with the number of models. On the markets of \Cref{sec:experiments}, with up to \(70\) models, tracing the full frontier using our method takes under half a second, whereas exhaustive search would be computationally infeasible at this scale.
\begin{algorithm}[H]
\caption{Tracing the arbitrage frontier}
\label{alg:frontier}
\small
\begin{algorithmic}[1]
\Require utility curves $\{u_i(x;\cdot)\}_{x\in\gD, i\in[N]}$, shadow prices
  $\lambda_1>\lambda_2>\dots>\lambda_L>0$, maximum number of sweeps $T$
\State $\vc\gets\mathbf 0$
\For{$k=1,\ldots,L$}\Comment{shadow prices decrease, so allocated budgets grow}
    \For{$t=1,\ldots,T$}
        \State $\vc^{\mathrm{old}}\gets\vc$
        \For{$i=1,\ldots,N$}
            \State $c_i\gets\argmax_{z\ge 0}\bigl\{U(c_1,\ldots,c_{i-1},z,c_{i+1},\ldots,c_N)-\lambda_k z\bigr\}$ \Comment{$U(\vc)$ as in \Cref{eq:indep_cost_objective}}
        \EndFor
        \If{$\vc=\vc^{\mathrm{old}}$} \textbf{break} \EndIf
    \EndFor
    \State $\vc^{(k)}\gets\vc$
\EndFor
\State \Return $\{\vc^{(k)}\}_{k=1}^L$ \Comment{the frontier: solve rate $U(\vc^{(k)})$ against total allocated budget $\sum_i c^{(k)}_i$}
\end{algorithmic}
\end{algorithm}
The KKT conditions characterizing the optimal allocation in \Cref{thm:KKT} yield several interesting implications discussed below. 
\paragraph{No further arbitrage opportunities.}
At the optimum, every funded model delivers the same marginal utility per dollar, \ie, 
\begin{align}\label{eq:KKT_interpretation}
        &\frac{\partial U(\vc^*)}{\partial c_i}
        =\E_{x\sim D}\!\left[
\underbrace{\frac{\partial \alpha_i(x,c_i^*)}{\partial c_i}}_{\substack{\text{Marginal log-failure reduction}\\\text{purchased per dollar for $x$}}}
\underbrace{e^{-\sum_{\ell=1}^N\alpha_\ell(x,c_\ell^*)}}_{\substack{\text{Probability of failure of}\\\text{the model pool on $x$}}}
\right]
        \;\begin{cases} =\lambda, & c_i^*>0,\\ \leq\lambda, & c_i^*=0. \end{cases}
\end{align} 
As a result, reallocation of the budget between funded models cannot raise utility,  \ie, no more arbitrage possibilities are available, which suggests the existence of a single consistent price system~\citep{ross1978simple,harrison1979martingales,dybvig1987arbitrage,varian1987arbitrage}.

\paragraph{Arbitrage against arbitrary pricing.} A growing line of recent work identifies shortcomings in current LLM pricing practices~\citep{artola2026overcharging,cai2025getting,saig2024incentivizing}, and proposes alternatives such as pay-per-character or pay-for-performance. Computational arbitrage offers a complementary perspective: the arbitrageur pays $1/\lambda$ dollars per unit of marginal utility from every funded provider, \ie, the effective price of marginal utility is equalized across providers, and providers whose marginal utility per dollar falls below $\lambda$ receive no revenue at all. This disciplines pricing regardless of the particular billing scheme each provider adopts. This mechanism parallels the role of arbitrage in goods markets: arbitrage pushes prices to their fundamental values~\citep{ross1976arbitrage,dybvig1987arbitrage,shleifer2000inefficient}, and constrains firms' ability to sustain discriminatory prices~\citep{malueg1994parallel,anderson1999international}.

\paragraph{Arbitrage values complementary specialization.}
Fix an allocation $\vc$ and the other models' performance curves.
Suppose provider $i$ makes a feasible local improvement
in its marginal utility per dollar:
\[
    \frac{\partial\tilde u_i(x;c_i)}{\partial c_i}
    =
    \frac{\partial\tilde\alpha_i(x,c_i)}{\partial c_i}
    e^{-\tilde\alpha_i(x,c_i)}
    =
    \frac{\partial\alpha_i(x,c_i)}{\partial c_i}
    e^{-\alpha_i(x,c_i)}
    +\varepsilon h_i(x),
    \qquad h_i(x)\geq0.
\]
The resulting change in its marginal contribution is
\[
    \frac{\partial\tilde U(\vc)}{\partial c_i}
    -
    \frac{\partial U(\vc)}{\partial c_i}
    =
    \varepsilon\,\E_{x\sim D}\left[
        h_i(x)\underbrace{e^{-\sum_{\ell\ne i}\alpha_\ell(x,c_\ell)}}_{\substack{\text{Probability of failure of}\\ \text{other models on $x$}}}
    \right].
\]
For a fixed average improvement $\E_{x\sim D}[h_i(x)]$,
targeting queries on which the other models are more likely
to fail yields a larger increase in the provider's marginal contribution. Furthermore, as competing models improve based on
the current allocation, this weight decreases on the queries
they address. Thus, the arbitrage allocation rule values \emph{complementary specialization}.

\section{Optimal deployment order}\label{sec:order}
The optimal allocation identifies how to distribute inference spending across providers to maximize utility under a given budget. Deploying this allocation raises a complementary question: in what order should the models be queried? In verifiable domains, the arbitrageur can stop as soon as an acceptable solution is found, so deployment order affects realized spending without changing the allocation’s success probability. Taking the allocation as fixed, we characterize cost-minimizing deployment orders and examine their implications for provider revenues and incentives.

Let $\pi$ be the permutation giving the deployment order.
%
We use the survival function identity to model the spend on each model; after spending $t$ dollars on model $i$ for solving sample $x$, the probability that it has not yet succeeded is $1-u_i(x;t)$. The expected spend on model $i$ for query $x$, conditional on reaching it, is therefore
\begin{align}\label{eq:ind_cost}
R_{i}(x;c_{i}):= \int_0^{c_{i}} \bigl(1-u_i(x;t)\bigr)\,\mathrm{d}t.
\end{align}
%
The expected spend for solving query $x$ is hence:
\begin{align}\label{eq:total_cost}
    R^{\pi}(x;\vc) = \sum_{i=1}^N
    R_{\pi(i)}(x;c_{\pi(i)})
    \underbrace{\prod_{j<i}\bigl(1-u_{\pi(j)}(x;c_{\pi(j)})\bigr)}_{\substack{\text{Probability that all}\\\text{ earlier models fail on $x$}}}.
\end{align}
The query-dependent cost-minimizing order then has a simple characterization:
\begin{restatable}{proposition}{propOrder}
    For a fixed allocation \(\mathbf{c}\) and instance \(x\), the order $\pi$ that minimizes
\(R^\pi(x;\mathbf{c})\) queries the models in decreasing order of acceptance probability
per expected dollar spent, \(\frac{u_i(x;c_i)}{R_i(x;c_i)}\). 
\end{restatable}
%
The per-query optimum provides a benchmark for deployment and motivates query-adaptive ordering when reliable performance predictions are available, paralleling results in LLM routing literature. Without such predictions, we seek a fixed order that minimizes expected cost across the query population. In the following, we show that finding the optimal deployment order in this case is NP-hard, using a reduction from the min-sum set cover problem.
\begin{restatable}{theorem}{propNPHard}
    Finding a fixed deployment order that minimizes $\E_{x\sim D}[R^\pi(x;\vc)]$ is NP-hard.    
\end{restatable}
Given this hardness result, we adopt a greedy ordering rule inspired by the 4-approximation algorithm for min-sum-set-cover~\citep{feige2004approximating}. Starting from an empty set, each time we append the model with the
largest reach-weighted coverage per expected dollar. Let $S_t$ be the set of models already queried up to step $t$ and let $\rho_{S_t}(x) = \prod_{i\in S_t} (1-u_i(x;c_i))$; then our greedy algorithm picks a model at step $t+1$ such that:
\begin{align}
    \pi(t+1)
    =
    \argmax_{i\notin S_t}
    \frac{\E_{x\sim D}[\rho_{S_t}(x)\,u_i(x;c_i)]}
    {\E_{x\sim D}\bigl[\rho_{S_t}(x)\,R_i(x;c_i)\bigr]}, 
    \qquad
    S_t=\{\pi(1),\dots,\pi(t)\},
    \label{eq:pop_order_greedy}
\end{align}
\ie, the model with the highest expected utility per expected dollar spent on yet unsolved samples.
\paragraph{Arbitrage incentivizes models to either be cheap generalists or specialists.}
Given the above greedy deployment order, model providers can have two strategies for maximizing their revenue.
First note that for every query the probability of reaching the $i$th model in the order, $\prod_{j<i}\bigl(1-u_{\pi(j)}(x;c_{\pi(j)})\bigr)$, is the probability that models earlier in the order fail to solve the query, and it decreases monotonically for models later in the sequence. As a result, there is a clear advantage for providers to appear early in the deployment order. To do so, they need to offer a competitive utility at a low price, as in the first iteration of the greedy algorithm the model with the highest \emph{average} utility per dollar will be selected. In later iterations, the models with the highest utility per dollar on \emph{unsolved} queries get selected.
In short, models should either simply be cheap enough to appear early in the deployment order or provide complementary strength, solving queries that remain unresolved.

\section{Reducing the cost of arbitrage}\label{sec:passk}
Computing the optimal allocation requires query-level utility curves. Estimating these curves with standard unbiased estimators requires many recorded attempts per model--query pair and since each attempt is a paid model call, collecting them can be prohibitively expensive. In this section, we adapt a distributional approach to obtain query-level empirical Bayes estimates from limited observations. 

We consider a repeated sampling strategy\footnote{Note that finding the optimal arbitrage strategy remains NP-hard under repeated sampling with integer sample counts; see \Cref{thm:repeated_sampling_hardness} in \Cref{app:proofs}.} and assume that $u_i(x;c_i) := \pass{i}{x}{k}$ where $k=c_i/q_i(x)$ with $q_i(x)$ representing the average cost per attempt of model $i$ on query $x$.
Following prior distributional approaches~\citep{kazdan2025efficient,schaeffer2025powerlaws}, we model one-shot success probabilities across queries using a model-specific Beta prior. Since this prior assigns no mass to exactly zero success probability, we augment it with a point mass at zero for queries that remain unsolvable by a given model under repeated sampling, \ie, 
\[
 \pass{i}{x}{1}
 \sim \pi_i\delta_0+(1-\pi_i)\text{Beta}(a_i,b_i),
\]
where $\delta_0$ is a point mass at zero and $\pi_i$ is the prior probability that a query has zero success probability for model $i$. To stabilize estimation from few attempts per query, we maximize the log marginal likelihood with an additional penalty $\frac{5}{2}\log(a_i+b_i)$~\citep{gelman2013bayesian}, which discourages the fitted Beta distribution from concentrating around a single success probability.
We fit the parameters by maximization marginal likelihood, \ie, 
\begin{align*}
 (\hat a_i,\hat b_i,\hat\pi_i)
 =\argmax_{a,b,\pi}\;(a+b)^{-5/2}\prod_{x\in\gD}
 \left[
 \pi\1\{s_{i,x}=0\}
 +(1-\pi)\binom{n_{i,x}}{s_{i,x}}
 \frac{\text{Beta}(a+s_{i,x},\,b+n_{i,x}-s_{i,x})} 
      {\text{Beta}(a,b)}
 \right],  
\end{align*}
where $n_{i,x}$ is the total number of attempts of model $i$ for sample $x$ and $s_{i,x}$ is the number of successful such attempts.

This population-level estimate is insufficient for computational arbitrage as the arbitrageur depends on the entire matrix of model-query utility curves $u_i(x;c_i)$ to find the optimal allocation. We therefore adapt the fitted model into a query-level empirical Bayes estimator.
For every model, the fitted distribution is used as a prior for the query-specific one-shot success probability pass@$1$. Given the observed attempts for model $i$ on query $x$, the posterior is $\hat\pi_{i,x}\delta_0+(1-\hat\pi_{i,x})\text{Beta}(\hat a_i+s_{i,x},\hat b_i+n_{i,x}-s_{i,x})$, where $\hat\pi_{i,x}$ is the posterior probability that model $i$ has zero success probability on query $x$, obtained by updating the fitted prior with the observed attempts. It is zero whenever at least one success is observed.
Hence, the resulting query-level pass@$k$ estimator is
\begin{align}\label{eq:pass_k_ours}
 \hatpass{i}{x}{k}
 =(1-\hat\pi_{i,x})\,
 \E_{\pass{i}{x}{1}\sim\text{Beta}(\hat a_i+s_{i,x},\hat b_i+n_{i,x}-s_{i,x})}
 \bigl[1-(1-\pass{i}{x}{1})^k\bigr].
\end{align}
In the next section, we show that arbitrage strategies based on our estimates achieve near-oracle solve rates while requiring orders of magnitude fewer model calls for calibration.


\section{Experiments}\label{sec:experiments}
\begin{figure}[t]
    \centering
    \includegraphics[width=\textwidth]{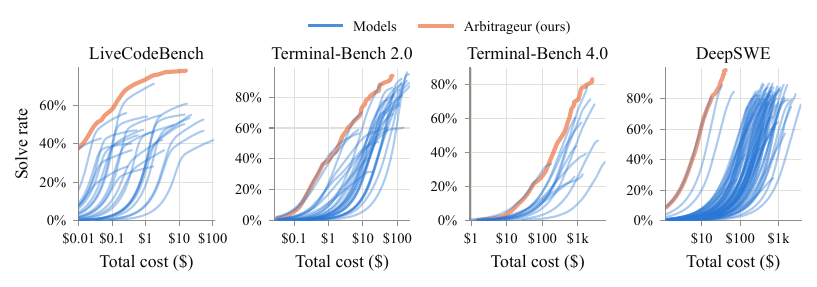}
    \caption{\textbf{Arbitrageur improves the cost--performance frontier.}
    Solve rate against total cost on four benchmarks. Each blue curve is
    one model funded on its own by repeated sampling with early stopping when a solution is found.
    The arbitrage gain is largest in the mid-budget regime, where models' strengths are most
    complementary. DeepSWE is the most concentrated of the four markets where one cheap model is the most cost-effective.
    } 
    \label{fig:frontier}  
\end{figure}

\begin{figure}[t]
    \centering
    \includegraphics[width=0.9\textwidth]{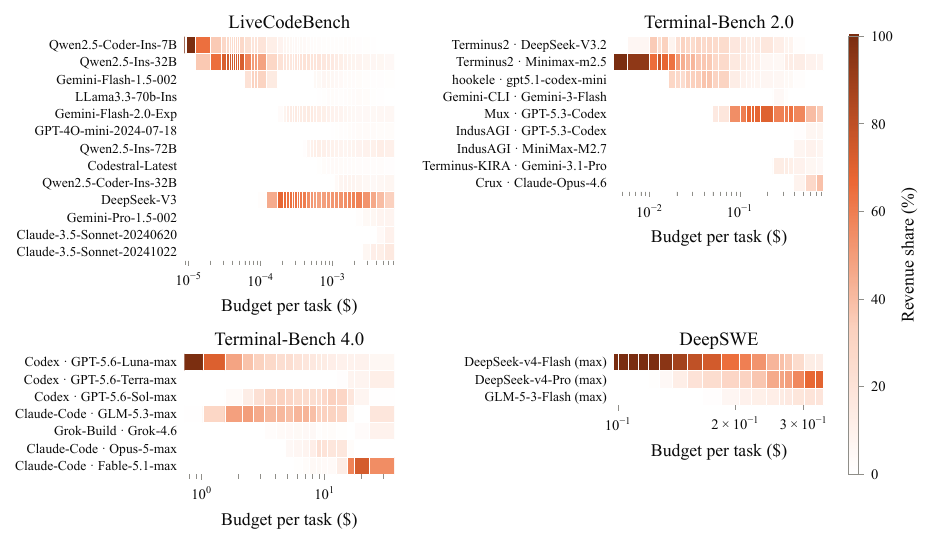}
    \caption{\textbf{Revenue spreads across providers as budgets increase.} Revenue share of each funded provider
    as a function of the budget per task under the arbitrageur's cross-validated allocation of
    \Cref{fig:frontier}, one row per system sorted by per-attempt price (cheapest on top). At lower budgets, revenue goes entirely to the
    cheaper models; as the budget grows, it spreads among more providers, and even at high
    budgets cheap models retain a significant share of the revenue. In DeepSWE, a few cheap configurations of one provider offer the highest utility per
dollar over most of the budget range.}  
    \label{fig:revshare}
\end{figure}  

\begin{figure}[t]  
    \centering
    \includegraphics[width=0.9\textwidth]{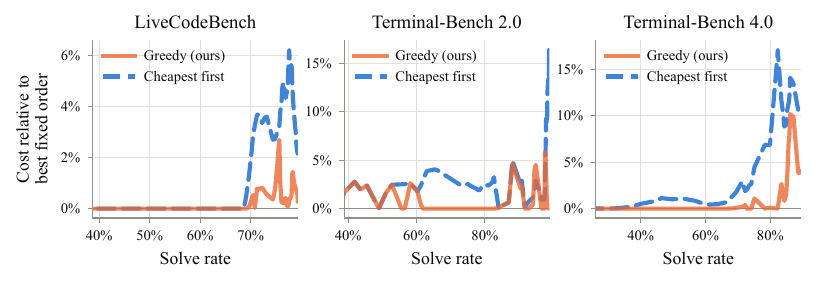}
    \caption{\textbf{The greedy deployment order is close to the best fixed order.}
    Expected spend of the arbitrageur's allocation when the funded models are queried in the greedy order of
    \Cref{eq:pop_order_greedy} (solid) or in increasing order of average cost per attempt (dashed), relative to
    the best fixed order of the funded models found by exact dynamic programming (which remains exponential in the number of funded models, the $0\%$ line), as a function of the
    solve rate. 
    The greedy order matches the best order at
    the median and stays within $10\%$ of it, whereas cheapest first spends up to $16\%$ more.}
    \label{fig:order} 
\end{figure}

\begin{figure}[t]
    \centering
    \includegraphics[width=\textwidth]{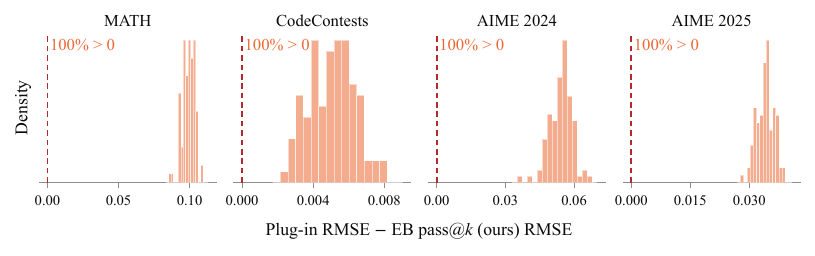}
    \caption{\textbf{Our estimator consistently outperforms the plug-in estimator.} Distribution over
    $100$ random calibration draws of the paired gap in pass@$k$ RMSE against the oracle reference,
    averaged over models and $k\leq 500$, with $n=8$ attempts per query--model pair, on MATH, CodeContests,
    AIME 2024 and AIME 2025; the two estimators are fitted on the same subsample
    of each draw. The dashed line marks no difference; mass to its right means our estimator is more accurate
    on that draw, and the annotation gives the share of such draws. The gap is positive on every draw on all benchmarks showing that our estimator consistently has lower RMSE than the plug-in estimator. Results for the MMLU benchmarks are shown in \Cref{fig:gap_app}.}
    \label{fig:gap}
\end{figure}

\begin{figure}[t]
    \centering
    \includegraphics[width=\textwidth]{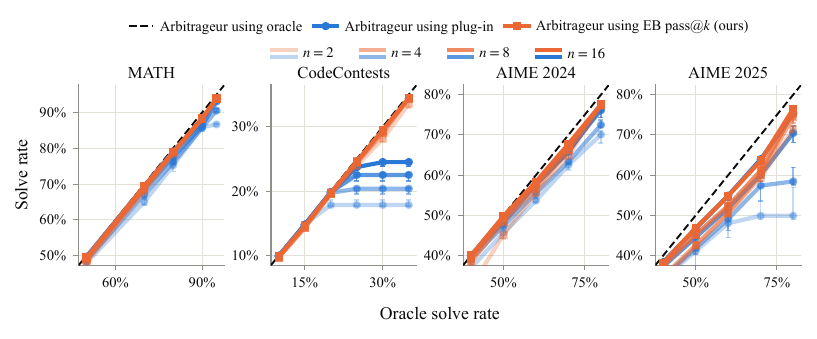}
    \caption{\textbf{An arbitrageur using our estimation method approaches the solve rate of an arbitrageur using the oracle utility curves which use several orders of magnitude more data.}
    On each benchmark of \Cref{fig:gap} we subsample $n\in\{2,4,8,16\}$ attempts per
    query--model pair, re-estimate the utility curves, run the arbitrage on the estimates
    and plot the achieved solve rate; the oracle uses all recorded attempts ($10{,}000$ per
    pair on the Monkey Business benchmarks, ${\sim}2{,}500$ on AIME) and is
    the diagonal. For each oracle solve rate we find the budget that the oracle needs to achieve that solve rate and use the same budget for other methods.} 
    \label{fig:calibration} 
\end{figure} 
In this part we evaluate our proposed arbitrage strategy, based on \Cref{thm:KKT}, together with the greedy deployment order, as in \Cref{eq:pop_order_greedy}, on four agentic and competitive programming benchmarks with recorded dollar costs for each trial. We further evaluate our estimator proposed in \Cref{sec:passk} on several benchmarks with large public records of repeated sampling.

\subsection{Arbitrage on real-world model markets}\label{subsec:exp:1}
We evaluate our arbitrage strategy on recent benchmarks spanning competitive-programming, agentic command-line and software-engineering tasks on real codebases. They evaluate many models over repeated trials, and, unlike most benchmarks, record what each trial costs, allowing us to study both cost--performance trade-offs and differences in model strengths. 

We use LiveCodeBench, Terminal-Bench~2.0, Terminal-Bench~4.0, and DeepSWE.\footnote{Refer to \Cref{app:benchmark_details} for detailed information about each benchmark.} We estimate the mean cost per attempt $q_i(x)$ for all model--query pairs using the reported cost per trial when available, and otherwise the recorded token counts scaled by public prices as of September 2026. 
More details on the experimental setup can be found in \Cref{app:experiments}.

%
For each model--query pair, we estimate the success probabilities under repeated sampling, where $c_i$ translates to $k=c_i/q_i(x)$ attempts of model $i$ on query $x$,\footnote{For non-integer $k$ the failure probability decays at a constant rate within an attempt.} using the unbiased estimator $\pass{i}{x}{k}:=1-\frac{\binom{n_{i,x}-s_{i,x}}{k}}{\binom{n_{i,x}}{k}}$~\citep{chen2021evaluating}, frozen past the recorded attempts.\footnote{With at most $5$--$10$ recorded attempts, held-out queries can only be scored within the recorded attempts. We therefore fit and evaluate the arbitrage strategy within the recorded attempts. Our pass@$k$ estimator proposed in \Cref{sec:passk}, whose value lies in predicting beyond the recorded attempts, is evaluated in \Cref{subsec:est_exp} on benchmarks with thousands of attempts per pair.}
The expected spend under repeated sampling that stops upon success is
calculated as in \Cref{eq:ind_cost}; for a portfolio, models are queried
in the greedy order of \Cref{sec:order} and the expected spend is calculated as in \Cref{eq:total_cost}. 
The arbitrageur's frontier is traced by sweeping the shadow price $\lambda$ over
a fixed grid of $86$ values; at each, the allocation and its deployment order
are cross-validated (leave-one-out on Terminal-Bench~2.0 and 4.0, $10$-fold on LiveCodeBench and DeepSWE). Total cost
is expected spend per query summed over the problems.

\Cref{fig:frontier} reports the solve rate of the arbitrageur and all the models alone, against their total spend, which shows that \textbf{the arbitrageur consistently improves the cost--performance frontier}. The largest gains occur at intermediate budgets, where combining models with complementary strengths is most valuable. It can be seen that DeepSWE is the most concentrated of the four markets: one cheap model is the most cost-effective arm over most of the budget range, so the arbitrageur tracks it and adds a more expensive model of the same provider only at the highest budgets; its gain there is accordingly the smallest. 

\Cref{fig:revshare} shows the revenue share of each model provider under the arbitrage strategy. At lower budgets, revenue goes entirely to the
cheaper models; as the budget grows, it spreads among more providers, and even at high
budgets cheap models retain a significant share of the revenue. These results suggest that \textbf{arbitrage can reduce market segmentation by directing revenue towards multiple providers, including smaller models that offer utility for a competitive price}.
\Cref{fig:order} compares the realized spend of our greedy order proposed in \Cref{sec:order} with the best fixed order of the funded models, computed by an exponential-time dynamic programming algorithm,\footnote{Note that exponential dynamic programming is possible here because only a handful of models receive non-zero allocations; in larger markets with more funded models, it can become computationally prohibitive.} and with the common heuristic of querying cheapest models first. \textbf{Our greedy order matches the best order at most budgets} and remains consistently close to it, unlike the cheapest-first order. On DeepSWE at most three models are funded and both ordering rules recover the best order.

\subsection{Arbitrage with estimated utility curves}\label{subsec:est_exp}
To evaluate the quality of our estimation method proposed in \Cref{sec:passk}, we need to determine how much of the solve rate an arbitrageur who uses our method to estimate the utility curves loses, compared to an arbitrageur who finds the true utility curves. To do so, we utilize benchmarks with many recorded attempts per model--query pair, allowing us to construct high-precision references for utility curves.
We use MATH (10 models, 128 problems) and CodeContests (5 models, 140 problems) from the Monkey Business benchmarks~\citep{brown2024largelanguagemonkeys} which provide $10,000$ recorded binary attempts per query--model pair as well as AIME 2024--2025, Global-MMLU-Lite~\citep{singh2025global}, and MMLU-Pro~\citep{wang2024mmlupro} (12 models, 30 problems per benchmark) where we use the $2,500$ recorded binary attempts per pair from~\citet{truong2026item}. None of these benchmarks record a price per attempt, so we price attempts by a compute proxy: model parameters times the model's mean generation length on the problem. Results for the MMLU benchmarks are reported in \Cref{app:utility_estimation}. 

We compare three estimation methods:
(i) the \emph{oracle} which estimates $\pass{i}{x}{k}$ based on all the available attempts using the unbiased estimator,
(ii) the \emph{plug-in} which estimates based on $n$ attempts per model-query pair, using the plug-in estimator $\hat p = s_{i,x}/n$ and extrapolates $1-(1-\hat p)^{k}$ to any $k$,
(iii) our \emph{empirical Bayes (EB)} estimator which estimates based on $n$ attempts per pair using \Cref{eq:pass_k_ours}. 

We first compare our method to the plug-in estimator in terms of the RMSE of the estimated pass@$k$ curves, using the oracle as the ground truth. \Cref{fig:gap} shows that \textbf{our estimator consistently has lower RMSE than the plug-in estimator}, suggesting that pooling information across queries (through the model-level prior) substantially improves utility estimation in the low-data regime.
We next visualize the solve rate of an arbitrageur using different estimators. \Cref{fig:calibration} summarizes the results where for each oracle solve rate, all methods spend the deployment budget the oracle requires for that rate, and we report medians over $10$ independent draws. It can be seen that \textbf{an arbitrageur using our estimation method tracks the oracle solve rate which uses up to orders of magnitude more data, and hence more cost}. On the CodeContests benchmark, the plug-in estimator often observes no successful attempts for a query and therefore assigns it \(\hat p=0\), regardless of \(k\). Its achievable solve rate is consequently limited to queries with at least one observed success. Our estimator avoids this failure mode by borrowing strength across queries through a model-specific prior.
Similar results for the MMLU benchmarks are reported in \Cref{fig:calibration_app}. 

\section{Further related work}

\textbf{Model cascading.} In model cascading, models are queried sequentially, typically in increasing order of cost, until a satisfactory solution is found~\citep{varshney2022model, wang2023tabi,chen2023frugalgpt,aggarwal2024automix,ramirez2024optimising}. Computational arbitrage addresses the problem of deciding which models should even be funded and how the inference budget should be distributed among them to maximize performance. Consequently, an optimal arbitrage strategy may exclude models that are dominated by the rest of the portfolio, and its cost-minimizing deployment order is not necessarily in increasing model cost. 

\textbf{Model routing.} Model routing assigns each query to one of the several available models, typically using predicted difficulty, quality, or cost to improve the cost--performance trade-off~\citep{shnitzer2024large,ding2024hybrid,lu2024routing,sakota2024fly}. Computational arbitrage addresses a complementary problem: rather than selecting a single model, it determines how to distribute inference spending across a portfolio of providers, potentially querying and repeatedly sampling from multiple models for the same query. Routing could nevertheless strengthen an arbitrage strategy by predicting query-level model performance. Conversely, research on model routing can adopt arbitrage strategies as benchmark for algorithm development. 


\textbf{Inference-time scaling.}
A line of work studies the scaling behavior of LLMs in various aspects including context length~\citep{xiong2024effective}, inference optimization~\citep{snell2024scalingtesttime}, and repeated sampling~\citep{brown2024largelanguagemonkeys,schaeffer2025powerlaws, kazdan2025efficient}.
%
Our work is most closely related to studies of repeated sampling as we estimate query-level scaling curves under repeated sampling to identify arbitrage strategies using limited data. Examining arbitrage strategies under other inference-time scaling methods such as chain-of-thought prompting~\citep{wei2022chain,wang2023selfconsistency}, majority voting~\citep{wan2025beacon}, and best-of-N~\citep{kalayci2025optimal,komiyama2025best} is an interesting direction for future work.

\textbf{Economics of LLMs.} 
A rapidly growing literature addresses the economic roles that LLMs can play in society~\citep{fish2024generative,boehmer2025generative,filippas2024large,raman2024steer,duetting2024mechanism}, examines cost-aware fine-tuning~\citep{laufer2024fine}, test-time training~\citep{chen2026leaderboard}, evaluation~\citep{erol2026cost,wang2024reasoning} and pricing and incentive problems in LLM markets, including opacity and potential overcharging~\citep{artola2026overcharging,artola2026testtime, sun2025coin,cai2025getting,saig2024incentivizing}. We study how computational arbitrage can itself reshape the economics of LLM markets by reallocating demand across them, affecting their revenue shares, and creating incentives for competitive pricing and complementary specialization.

\section{Conclusion}
In this work we take a step towards making computational arbitrage practical in large AI model markets. We identify a natural condition under which the optimal arbitrage strategy can be computed efficiently. At this optimal allocation, every funded model provides the same marginal utility per dollar, directing revenue toward providers that are either broadly cost-effective or specialized in queries that the rest of the portfolio is likely to fail on. The deployment order reinforces these incentives: inexpensive models with strong average performance tend to be queried early, while complementary specialists remain valuable later in the sequence. Together, these findings suggest that arbitrage can discipline pricing while rewarding specialization in AI model markets.
%


\section*{Acknowledgements}
We thank Florian Dorner, Ana-Andrea Stoica, and Yatong Chen for helpful discussions and valuable feedback on this work.

\bibliographystyle{plainnat}
\bibliography{refs}

\appendix
\crefalias{section}{appendix}
\crefalias{subsection}{subappendix}
\crefalias{subsubsection}{subsubappendix}

\section{Additional examples}\label{app:examples}

\begin{example}[Diversification despite pointwise dominance]\label{app:example_1}
Suppose two models have acceptance probabilities
\[
\Pr_{y\sim p_1(\cdot\mid x,c)}[A(x,y)=1]=1-(1+c)^{-1},
\qquad
\Pr_{y\sim p_2(\cdot\mid x,c)}[A(x,y)=1]=1-(1+c)^{-9/10}
\]
on every query $x$, and their outputs are conditionally independent given the query. Model 1 strictly dominates model 2: at every positive budget, it has a higher acceptance probability on every query.

With a total budget of $C=2$, allocating everything to the dominant model gives
\[
\Pr_{y\sim p_1(\cdot\mid x,2)}[A(x,y)=1]=\frac{2}{3}.
\]
Splitting the budget equally instead gives, with independent outputs $y_i\sim p_i(\cdot\mid x,1)$,
\[
\begin{aligned}
    \Pr[A(x,y_1)=1\text{ or }A(x,y_2)=1]
    &=1-\Pr[A(x,y_1)=0]\Pr[A(x,y_2)=0]\\
    &=1-2^{-19/10}
    \approx0.732
    >\frac{2}{3}.
\end{aligned}
\]
These acceptance probabilities are the same for every query, so the comparison holds under any query distribution $D$. Thus, even allocating the entire budget to a model that dominates every other model on every query can be suboptimal. Diminishing returns to spending on the dominant model can make diversification beneficial without complementary strengths across queries.
\end{example}

\begin{example}[Diversification through complementary strengths]\label{app:example_2}
Suppose we have two models and the query space is partitioned into two equally
likely disjoint regions
\(\gX=\gX_1\sqcup \gX_2\), with \(D(\gX_i)=1/2\). Model \(i\)
performs better on \(\gX_i\), and its utility is
\[
    \Pr[A(x,y)=1]
    =
    \begin{cases}
        1-e^{-ac}, & x\in \gX_i,\\
        1-e^{-bc}, & x\notin \gX_i,
    \end{cases}
    \qquad a>b>0.
\]
Thus, every model's performance increases with its allocated budget on
every query, but it improves faster on queries in its specialty. Assume further that the models' outputs are independent conditional on the query. Allocating the entire budget $C$ to either model gives the average acceptance probability of:
\[
    \Pr_{x\sim D, y\sim p_i(.\mid x,C)}[A(x,y)=1]
    =
    1-\frac{e^{-aC}+e^{-bC}}{2}.
\]
In contrast, if we allocate \(c_i=C/2\) to every model, for every query we have that
\[
    \Pr[A(x,y_1)=1 \text{ or } A(x,y_2)=1]
    =
    1-e^{-\frac{C}{2}(a+b)}
    >
    1-\frac{e^{-aC}+e^{-bC}}{2},
\]
where the strict inequality follows from the arithmetic--geometric mean
inequality. Hence, at the same spend \(C\) on every query,
a diversified portfolio can strictly outperform every single model.
\end{example}

\section{Proofs}\label{app:proofs}
\thmNPhard*
\begin{proof}\label{proof:nphard}
    We use a standard reduction from Maximum Coverage (MC), which is NP-hard~\citep{karp1972reducibility}. In MC, given a collection of sets $S:=(S_1,\dots,S_N)$, the objective is to select at most $B$ sets such that the number of covered elements of $\gX$ is maximized.
    
    Given an instance of MC, let $D$ be a uniform distribution over $\gX$. For each set $S_i$, we construct model $i$ such that
    \[
    \Pr_{y_i\sim p_i(\cdot\mid x,c_i)}
    \left[A(x,y_i)=1\right]
    =
   \1\{x\in S_i\}\1\{c_i\geq1\},
    \]
   and we set \(C=B\).

    Let $\vc:=(c_1, \dots, c_N)$ be an optimal solution to our problem. Let $I(\vc):=\{i:c_i\geq 1\}$. Model $i$ succeeds deterministically on $x\in S_i$ whenever a budget of at least 1 is assigned to it, and never succeeds outside $S_i$. Therefore, the portfolio succeeds on $x$ exactly when $x\in\bigcup_{i\in I(\vc)}S_i$, hence we have that
    

    \[
    U(\vc)=\frac{\lvert\bigcup_{i\in I(\vc)}S_i\rvert}{\lvert\gX\rvert}.
    \]
    %

    Now for any such optimal solution to our problem we construct a solution to MC where a set $S_i$ is selected whenever $i\in I(\vc)$. Note that since each $i\in I(\vc)$ receives at least one unit of budget, $\lvert I(\vc)\rvert\leq \sum_{i\in I(\vc)} c_i\leq \sum_{i\in [N]} c_i \leq B$.

    Conversely, given a solution to MC, we construct a solution to our problem by setting $c_i=1$ whenever $S_i$ is selected and $c_i=0$ otherwise. Our utility equals the fraction of elements covered by the selected sets, thus maximizing $U(\vc)$ is equivalent to MC up to the constant factor $\frac{1}{|\gX|}$.
    Given both directions of the proof, our problem is NP-hard.  
\end{proof}

%



\thmKKT*
\begin{proof}
We first show that under \Cref{ass:logconvex} the objective is monotone and jointly concave over the budget simplex \(\Delta:=\{\vc\in\R_{\geq 0}^N:\sum_{i\in[N]} c_i \leq C\}\). As a result, every KKT point is a global optimum.

We start by proving monotonicity and joint concavity.
Recall that $\alpha_i(x;c) :=-\log (1-u_i(x;c))$, and that we can rewrite the objective function in \Cref{eq:cost_objective} as:
\[
U(\vc) = \E_{x\sim D}[1-\exp(-\sum_{i=1}^N \alpha_i(x;c_i))],
\]
where by assumption, each $\alpha_i(x;c_i)$ is non-decreasing and concave, hence, $\sum_{i=1}^N \alpha_i(x;c_i)$ is jointly concave in $\vc$. Furthermore, $g\mapsto 1-\exp(-g)$ is strictly increasing and concave in $g$, and hence the composition $1-\exp(-\sum_{i=1}^N \alpha_i(x;c_i))$ is also jointly concave in $\vc$.
Taking expectation preserves concavity, hence, $U(\vc)$ is jointly concave in $\vc$. Monotonicity is immediate from each $\alpha_i(x;c_i)$ being non-decreasing. 
Since each \(\alpha_i(x,\cdot)\) is differentiable, the portfolio utility for every fixed \(x\), 
\[
1-\exp\left(-\sum_{i=1}^N\alpha_i(x,c_i)\right)
\] is continuous in \(\vc\). Since \(D\) has finite support, \(U(\vc)\) is a finite weighted sum of these continuous functions and is therefore continuous. Since \(\Delta\) is compact, the extreme value theorem guarantees that \(U\) attains a global maximum on \(\Delta\).

The feasible set \(\Delta\) is convex and described by affine constraints. For \(C>0\), the allocation \(c_i=C/(2N)\) is strictly feasible, hence Slater’s condition holds. Therefore, the KKT conditions are both necessary and sufficient~\citep{boyd2004convex}. As a result,
    a global maximizer is characterized by the existence of $\lambda\geq 0$ such
    that $\lambda(C-\sum_{i=1}^N c_i^*)=0$, and for all $i\in[N]$,
    \begin{align}
        &\frac{\partial U(\vc^*)}{\partial c_i}
        =\E_{x\sim D}\!\left[
    \underbrace{\frac{\partial \alpha_i(x,c_i^*)}{\partial c_i}}_{\substack{\text{Marginal log-failure reduction}\\\text{purchased per dollar for $x$}}}
    \underbrace{e^{-\sum_{\ell=1}^N\alpha_\ell(x,c_\ell^*)}}_{\substack{\text{Probability of failure of}\\\text{the model pool on $x$}}}
\right]
        \;\begin{cases} =\lambda, & c_i^*>0,\\ \leq\lambda, & c_i^*=0. \end{cases}
\end{align} 

In practice, we maximize \(U(\mathbf c)-\lambda\sum_i c_i\) over \(\mathbf c\geq0\) using coordinate ascent, sweeping the shadow price \(\lambda\) over a grid. Refer to \Cref{alg:frontier} for pseudocode of the algorithm.

\end{proof}

\propOrder*
\begin{proof}
Consider two models \(a,b\) occupying adjacent positions in \(\pi\), reached with
common probability \(\rho\) (all earlier models failed), and abbreviate
\(u_a:=u_a(x;c_a)\) and \(R_a:=R_a(x;c_a)\). Their contribution to
\(R^\pi(x;\vc)\) is \(\rho\bigl[R_a+(1-u_a)R_b\bigr]\) if \(a\)
precedes \(b\) and \(\rho\bigl[R_b+(1-u_b)R_a\bigr]\) otherwise. Every position
before the pair is unchanged, and every position after the pair sees the same
combined survival probability \((1-u_a)(1-u_b)\) regardless of the pair's internal
order, so a swap changes only this pair's contribution. Placing \(a\) first is
weakly better iff \(R_a+(1-u_a)R_b\le R_b+(1-u_b)R_a\), i.e.\ iff
\(R_a u_b\le R_b u_a\), or equivalently
    $\frac{u_a}{R_a}\ge\frac{u_b}{R_b}$.
So by an exchange argument the total cost is minimized by sorting the models in
decreasing order of their expected utility per cost on every query.
\end{proof}

\propNPHard*
\begin{proof}
We reduce from minimum-sum set cover. Let
$\mathcal{X}=\{x_1,\dots,x_m\}$ be a universe and let
$\mathcal{S}=\{S_1,\dots,S_N\}$ be a collection of sets whose union is
$\mathcal{X}$. For a permutation $\pi$ of the sets, define the cover time
of element $x$ by
\[
    T_\pi(x)
    :=
    \min\{t:x\in S_{\pi(t)}\}.
\]
The minimum-sum set-cover problem asks for a permutation minimizing
\[
    \sum_{x\in\mathcal{X}} T_\pi(x),
\]
and is NP-hard~\citep{feige2004approximating}.

Construct a model-ordering instance with one query for every
$x\in\mathcal{X}$, drawn uniformly, and one model $p_i$ for every set
$S_i$. Set $q_i(x)=c_i=1$ and let model $p_i$ succeed deterministically
exactly on the queries contained in $S_i$:
\[
    p_i(x)
    =
    \begin{cases}
        1, & x\in S_i,\\
        0, & x\notin S_i.
    \end{cases}
\]
We interpret the endpoint cases in the survival-function formulation
directly. If $x\notin S_i$, then
\[
    R_i(x;1)
    =
    \int_0^1 (1-p_i(x))^t\,\mathrm{d}t
    =
    1,
\]
whereas if $x\in S_i$, then
\[
    R_i(x;1)
    =
    \int_0^1 (1-p_i(x))^t\,\mathrm{d}t
    =
    0.
\]
For a fixed query $x$, every model preceding the first set containing
$x$ fails and contributes one unit of cost. The first model whose set
contains $x$ succeeds and terminates the cascade. Consequently,
\[
    R^\pi(x;\mathbf{c})
    =
    T_\pi(x)-1.
\]
Taking the expectation under the uniform query distribution gives
\[
    \E_{x\sim D}[R^\pi(x;\mathbf{c})]
    =
    \frac{1}{m}\sum_{x\in\mathcal{X}}T_\pi(x)-1.
\]
Therefore, an algorithm that computed the optimal model
order would also solve minimum-sum set cover. Hence, choosing the
cost-minimizing fixed model order is NP-hard.
\end{proof}

\begin{theorem}\label{thm:repeated_sampling_hardness}
Finding an optimal arbitrage policy is NP-hard, even under repeated sampling strategy.
\end{theorem}

\begin{proof}
We use the same reduction from Maximum Coverage as in the general hardness proof. Given sets $S_1,\dots,S_N\subseteq\gX$ and a budget of $B$ sets, let $D$ be uniform over $\gX$, set $C=B$, and assign each model unit cost per attempt, $q_i(x)=1$. Define its one-shot success probability by
\[
 \pass{i}{x}{1}=\1\{x\in S_i\}.
\]
For $k_i\in\mathbb{Z}_{\geq0}$ independent attempts, its success probability is
\[
 \pass{i}{x}{k_i}
 =\begin{cases}
 0, & k_i=0,\\
 1-(1-\pass{i}{x}{1})^{k_i}, & k_i\geq1,
 \end{cases}
 =\1\{x\in S_i\}\1\{k_i\geq1\}.
\]
These deterministic outcomes satisfy conditional independence. Since $c_i=k_i$, a feasible allocation satisfies $\sum_i k_i\leq B$. Writing $I(\mathbf{k})=\{i:k_i\geq1\}$, its utility is
\[
 U(\mathbf{k})
 =\frac{\bigl|\bigcup_{i\in I(\mathbf{k})}S_i\bigr|}{|\gX|},
 \qquad |I(\mathbf{k})|\leq B.
\]
Thus every feasible allocation selects at most $B$ sets with exactly the same normalized coverage. Conversely, any selection of at most $B$ sets is realized by assigning one attempt to each selected model and zero to the others. Repeated attempts on a selected model yield no additional utility. The construction and both mappings are polynomial in the Maximum Coverage input size, so an optimal repeated-sampling allocation yields an optimal Maximum Coverage solution.
\end{proof}

\section{Additional details on utility estimation methods}

\textbf{Sampling strategy to estimate pass@k.}
An unbiased and stable estimator of $\pass{i}{x}{k}$ was proposed by~\citet{chen2021evaluating}:
    $\hatpass{i}{x}{k} := 1-\frac{\binom{n_{i,x}-s_{i,x}}{k}}{\binom{n_{i,x}}{k}}$,
where $n_{i,x}$ is the total number of attempts of model $i$ for query $x$ and $s_{i,x}$ is the number of successful such attempts; note that the estimator above requires $n_{i,x}\geq k$ attempts. For large values of $k$, it requires many generations per problem which is costly. Therefore, in the next part we adopt a distributional strategy~\citep{schaeffer2025powerlaws,kazdan2025efficient} to estimate $\pass{i}{x}{k}$.

\textbf{Distributional strategy to estimate pass@k.}
Since attempts are i.i.d., $\pass{i}{x}{k}$ can be written as
    $\pass{i}{x}{k} = 1-(1-\pass{i}{x}{1})^k$.
Hence, to estimate $\pass{i}{x}{k}$, it suffices to estimate $\pass{i}{x}{1}$. \citet{schaeffer2025powerlaws} propose a distributional strategy estimator where they first select a prior distribution on $\pass{i}{x}{1}$, 
they discretize this distribution according to the sampling resolution,
and fit the distribution's parameters using MLE under the discretized distribution's PMF. \citet{kazdan2025efficient} further show that this discretization introduces bias and propose a strategy which does not require discretization. 

\section{Additional experimental details}\label{app:experiments}
\subsection{Details of the performance-target experiment}\label{app:fig_1}

We estimate the success probabilities $u_i(x;c_i)$ under repeated sampling using the unbiased estimator of \citet{chen2021evaluating} within the recorded attempts and from the empirical Bayes estimator of \Cref{sec:passk} beyond that. 
For every model--query pair, we find the expected spend under repeated sampling that stops upon success using the survival function $\int_0^{c_i}(1-u_i(x;t))\,\mathrm{d}t$. We use the same budget for each query, find the probability that the query is solved using that budget, find the average solve rate over the dataset, and increase the budget until the solve rate reaches the target. Budgets are swept over $600$ log-spaced points from $\$10^{-4}$ to $\$10$ per task.

\subsection{Benchmark details}\label{app:benchmark_details}

\emph{Terminal-Bench~2.0}\footnote{\href{https://hub.harborframework.com/datasets/terminal-bench/terminal-bench-2/latest}{Terminal-Bench~2.0}} is a benchmark of $88$ shared agentic software engineering tasks. Of its submissions, $75$ agent--model systems publish per-trial
artifacts.\footnote{\href{https://huggingface.co/datasets/harborframework/terminal-bench-2-leaderboard}{Terminal-Bench~2.0 leaderboard submission}} We keep the $28$ that record either a per-trial dollar cost or usable token counts.
Each model--query pair has at most $5$ recorded attempts with binary outcomes. We estimate the mean cost per attempt for all model--query pairs using the reported cost per trial when available, and otherwise the recorded input, output, and cached token counts scaled by the LiteLLM price database as of September 2026.\footnote{\href{https://github.com/BerriAI/litellm/blob/3d0fd127d5a151f7f094462b16ac9c2a01a047b6/model_prices_and_context_window.json}{LiteLLM price database snapshot, September 17, 2026}}

\emph{LiveCodeBench}\footnote{\href{https://github.com/LiveCodeBench/submissions}{LiveCodeBench leaderboard submissions}}~\citep{jain2024livecodebench}
consists of $73$ models on $713$ shared problems, however, only $36$ models record more than one attempt out of which $10$ models are  reasoning models which do not record reasoning traces so their token cost is not recoverable. We are hence left with $24$ models.
Each attempt is treated as a separate request and its input
is billed without caching.
We use the
model-specific tokenizer when publicly available and \texttt{cl100k\_base}\footnote{\href{https://tiktokenizer.vercel.app/?model=cl100k_base}{\texttt{cl100k\_base} BPE tokenizer}} as
a proxy otherwise. In this regard note that Terminal-Bench trials are multi-turn agent runs whose billed context is not recoverable from the stored logs, so systems that report neither cost nor token counts are excluded rather than priced based on proxy tokenizers. Prices are taken from the same dated LiteLLM and
OpenRouter snapshots as Terminal-Bench; historical list prices are used for models absent from
both snapshots.


\emph{Terminal-Bench~4.0}\footnote{\href{https://github.com/harbor-framework/terminal-bench/tree/main/leaderboard}
{Terminal-Bench~4.0 leaderboard}} provides $12$ agent--model systems on $66$
shared tasks with $5$ recorded attempts per (system, task). Here the leaderboard records the dollar cost of every
attempt directly, so no token pricing is needed; one submission whose attempts all report a cost of zero is excluded, as a
free arm would be infinitely cheap to the arbitrageur. 

\emph{DeepSWE}\footnote{\href{https://deepswe.datacurve.ai}
{DeepSWE v1.1 public trials}} provides $70$ (model, reasoning-effort) configurations of $28$ base models run under a
common agent harness on $111$ shared software-engineering tasks, with $4$ attempts per pair (a handful of pairs have
fewer) and a recorded dollar cost per attempt. As in the main text, every configuration is a distinct arm.

For all benchmarks, we treat each distinct agent--model pair or model--reasoning-effort configuration as a separate model.

\subsection{Experimental details of utility curve estimation}\label{app:utility_estimation}
For the Monkey Business benchmarks~\citep{brown2024largelanguagemonkeys}, the dollar cost of an attempt is not available so we use the number of FLOPs of one generation as the compute proxy which is calculated as the model's parameter count times the mean number of tokens it generates for that question~\citep{kaplan2020scaling}. The AIME 2024--2025, Global-MMLU-Lite and MMLU-Pro data are the test-time scaling experiments of \citet{truong2026item}: we take the outcomes from their response tensor\footnote{\url{https://huggingface.co/datasets/stair-lab/irsl_testtime_resmat2}} and the length of every response from the accompanying per-response release,\footnote{\url{https://huggingface.co/datasets/stair-lab/denoise_eval_query}} which records the number of generated tokens. We price an attempt on a question at the model's active parameter count times the mean number of tokens it generates on that question.

We report the estimation error of our method on the MMLU benchmarks in \Cref{fig:gap_app}, complementing the four benchmarks in \Cref{fig:gap}, and repeat the evaluation of \Cref{fig:calibration} on Global-MMLU-Lite~\citep{singh2025global} and MMLU-Pro~\citep{wang2024mmlupro}. Results are in \Cref{fig:calibration_app}. For both, we use the $30$ questions per benchmark and the $\sim2{,}500$ recorded binary attempts per model--question pair released by \citet{truong2026item} for $12$ open-weight reasoning models. As these models are self-hosted, the per-attempt cost is the compute proxy above. Note that only the cost ratios matter to the arbitrageur since scaling every cost by the same constant rescales the budget without changing the allocation.


\begin{figure}[t]
    \centering
    \includegraphics[width=0.8\textwidth]{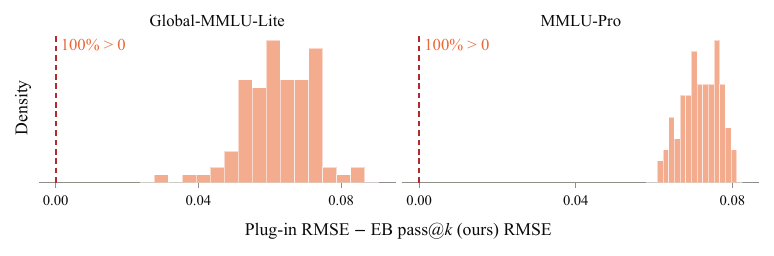}
    \caption{\textbf{Our estimator consistently outperforms the plug-in estimator.} Distribution over
    $100$ random calibration draws of the paired gap in pass@$k$ RMSE against the oracle reference,
    averaged over models and $k\leq 500$, with $n=8$ attempts per query--model pair, on Global-MMLU-Lite and MMLU-Pro; the two estimators are fitted on the same subsample
    of each draw. The dashed line marks no difference; mass to its right means our estimator is more accurate
    on that draw, and the annotation gives the share of such draws. The gap is positive on every draw on all benchmarks showing that our estimator consistently has lower RMSE than the plug-in estimator.}
    \label{fig:gap_app}
\end{figure}

\begin{figure}[t]
    \centering
    \includegraphics[width=0.8\textwidth]{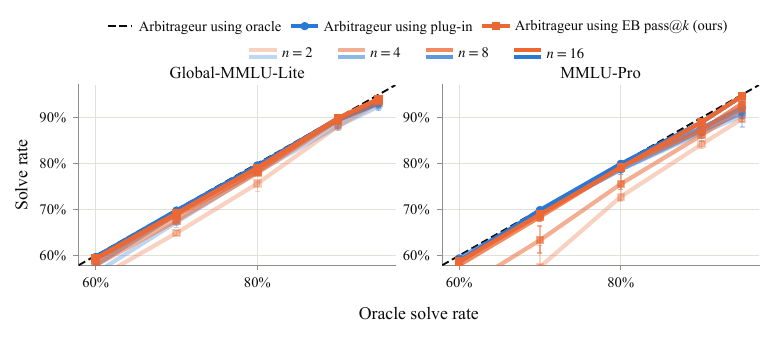}
    \caption{\textbf{An arbitrageur using our estimation method tracks the solve rate of an arbitrageur using the oracle utility curves which are calculated using several orders of magnitude more data.}
    On each of the MMLU benchmarks, we subsample $n\in \{2,4,8,16\}$ attempts per
    query--model pair, re-estimate the utility curves, run the arbitrage on the estimates
    and plot the achieved solve rate; the oracle uses all the $2{,}500$ recorded attempts and is
    the diagonal. For each oracle solve rate we find the budget that the oracle needs to achieve that solve rate and use the same budget for the other methods.}   
    \label{fig:calibration_app}
\end{figure}


\end{document}